\documentclass[12pt,journal,onecolumn]{IEEEtran}
\usepackage{amsmath,amsfonts,amssymb,txfonts,mathrsfs}
\usepackage{epsf,graphics,graphicx,subeqnarray,subfigure,subfigmat}
\usepackage{times,color,colordvi,multirow}
\usepackage[sort]{cite}

\newcommand{\ben}{\begin{enumerate}}
\newcommand{\een}{\end{enumerate}}
\newcommand{\be}{\begin{equation}}
\newcommand{\ee}{\end{equation}}
\newcommand{\bea}{\begin{eqnarray}}
\newcommand{\eea}{\end{eqnarray}}
\newcommand{\bsea}{\begin{subeqnarray}}
\newcommand{\esea}{\end{subeqnarray}}
\newcommand{\bit}{\begin{itemize}}
\newcommand{\eit}{\end{itemize}}

\newtheorem{lemma}{\bf Lemma}[section]

\newtheorem{theorem}{\bf Theorem}[section]

\newenvironment{proof}{\noindent \underline{\bf Proof:} }{\hfill \rule{2mm}{2mm}}
\newcommand{\bpro}{\begin{proof}}
\newcommand{\epro}{\end{proof}}

\newcommand{\norm}[1]{\left\lVert#1\right\rVert}
\newcommand{\inner}[2]{\langle #1, #2 \rangle}
\DeclareMathOperator*{\argmax}{arg\,max}
\DeclareMathOperator*{\argmin}{arg\,min}

\title{Risk Bounds for Learning via Hilbert Coresets}
\author{
    \IEEEauthorblockN{Spencer Douglas\IEEEauthorrefmark{1}, Piyush Kumar\IEEEauthorrefmark{1}, R.K. Prasanth\IEEEauthorrefmark{1}}\\
    \IEEEauthorblockA{\IEEEauthorrefmark{1}Systems \& Technology Research \\ 600 West Cummings Park, Suite 6500, Woburn, MA 01801, USA}
 
}

\begin{document}
\maketitle

\begin{abstract}
We develop a formalism for constructing stochastic upper bounds on the expected full sample risk for supervised classification tasks via the Hilbert coresets approach within a transductive framework. We explicitly compute tight and meaningful bounds for complex datasets and complex hypothesis classes such as state-of-the-art deep neural network architectures. The bounds we develop exhibit nice properties: i) the bounds are non-uniform in the hypothesis space $\mathcal{H}$, ii) in many practical examples, the bounds become effectively deterministic by appropriate choice of prior and training data-dependent posterior distributions on the hypothesis space, and iii) the bounds become significantly better with increase in the size of the training set. 
We also lay out some ideas to explore for future research. 
\end{abstract}

\section{Introduction}\label{intro}

Generalization bounds for learning provide a theoretical guarantee on the performance of a learning algorithm on unseen data. The goal of such bounds is to provide 
control of the error on unseen data with pre-specified confidence. In certain situations, such bounds may also help in designing new learning algorithms. With the great success of deep neural networks (DNNs)
on a variety of machine learning tasks, it is natural to try to construct theoretical bounds for such models. Recent years have seen several notable efforts to construct these bounds. 

One classical approach, that of Probably-Approximately-Correct (PAC) learning, aims to develop generalization bounds that hold (with high probability) {\it uniformly} over all hypotheses $h$ (parametrized with parameters $\theta$) in some hypothesis class $\mathcal{H}$ and over all data distributions $D$ in some class of distributions $\mathcal{D}$. The well-known VC-dimension bounds \cite{Vapnik} fall in this category. Since these are uniform bounds, they are \emph{worst case} and very general. However, the flip side is that they tend to be very loose and not useful in practical situations. For example, the VC dimension of neural networks is typically bounded by the number of parameters of the network. Since modern DNNs have millions of parameters, these uniform VC dimension bounds are typically not useful at all. 

Better prospects for developing useful and non-trivial bounds arise when one develops \emph{data-dependent} and \emph{non-uniform} bounds. For example, there has been lot of recent work on bounds based on \emph{Rademacher complexity} \cite{Neyshabur2015, Bartlett2017}. These bounds are dependent on the particular data distribution $D$ ($\in \mathcal{D}$) but still depend uniformly on the entire hypothesis class $\mathcal{H}$. Recent work on computing these bounds for modern DNNs still gives rise to bounds that may not be useful for at least some classes of modern DNNs \cite{Dziugaite2017}. 

Another data-dependent approach to developing useful generalization bounds is that of the \emph{PAC-Bayesian} framework \cite{McAllester1999, Langford2002, Seeger2003, Catoni2007}. In this case, one starts with some 'prior' distribution $\pi_0$ (that is independent of the data) over the hypothesis space $\mathcal{H}$, and then constructs a 'posterior' distribution $\nu$ over $\mathcal{H}$ after training the model on some training dataset. Note that the 'prior' and 'posterior' do not refer to the corresponding terms in Bayesian probability in the strict sense.  Here, the posterior distribution $\nu$ can be thought of as being determined by the training data and the training algorithm in general. The PAC-Bayes framework studies an averaging of the generalization error according to the posterior $\nu$ (the so-called Gibbs classifier). The bounds thus obtained are valid for all posterior distributions $\nu$ over the hypothesis space $\mathcal{H}$, but the bound itself depends on the posterior $\nu$. Thus, these bounds are \emph{non-uniform} in the hypothesis space as well. In other words, the precise numerical bounds are valid for parameters in the neighborhood of those learned by the \underline{actual} learning algorithm. Thus, these have the potential to provide tight and useful bounds for complex models like DNNs. 

A point to note is that the bulk of the work in statistical learning theory within each of the above approaches has dealt with providing bounds for the \emph{inductive} learning setting. In this case, the goal is to learn a labeling rule from a finite set of labeled training examples, and this rule is then used to label new (test) examples.  In many realistic situations, however, one has a \emph{full sample}, which consists of a training set of labeled examples, together with an unlabeled set of points which needs to be labeled.This is the \emph{transductive} setting, where one is interested only in labeling the unlabeled set as accurately as possible. In this case, it is plausible that if one leverages the information of the test data points appropriately, one may be able to derive tighter risk bounds. Transductive learning bounds are valid for the particular dataset at hand. If the dataset is augmented with new data points, the bounds have to be computed again in general. In this work, we show that it is possible to get tight bounds on the full sample risk by using the method of Hilbert coresets, a coreset defined in terms of a Hilbert norm on the space of loss functions (a more precise description will follow later). This work builds on the work of Campbell and Broderick \cite{Campbell2017, Campbell2018} and applies it to the transductive learning framework. Our principal contributions in this paper are as follows:

\bit
\item We compute non-vacuous and tight upper bounds on the \emph{expected full sample risk} for multi-class classification tasks for a variety of state-of-the-art hypothesis classes (DNN architectures) and datasets, and compare these results with that obtained from explicit training algorithms. Our bounds are valid for any data distribution that lies in a class of distributions $\mathcal{D}$. In particular, we study the class of distributions $\mathcal{D}$: \emph{Convex hull of a probability simplex} $\mathcal{D}_{CH}^{(K)}$. The calculation of the bounds for the expected full sample risk requires knowledge of the full dataset along with the labels. 

\vspace{0.3cm}

\item The developed bounds have nice properties. First, these bounds are \emph{non-uniform} in the sense that although the bounds are valid for all distributions over $\mathcal{H}$, the precise bound depends on particular distribution over $\mathcal{H}$, i.e. the precise neighborhood of the parameters $\theta_{\star}$ learned by the training algorithm. 

\vspace{0.3cm}

\item Second, although the bounds we develop technically apply to a \emph{stochastic} DNN, appropriate choice of data-independent priors and data-dependent posteriors in many practical examples gives rise to Gibbs classifiers that are highly clustered around the classifier learned by the original deterministic DNN (with parameters $\theta_{\star}$ learned by the training algorithm, see results in Section \ref{results}). Thus, although there is no guarantee, the bounds can become effectively \emph{deterministic} in many practical cases.

\vspace{0.3cm}

\item Finally, the numerical bounds \emph{decrease} with an increase in the
number of training samples. This behavior has been tested for a variety of DNNs and holds up across changes in depth and width of the DNN as well as 
different mini-batch size used for training. This is in contrast with certain existing bounds which may even \emph{increase} with an increase in the number of
training samples \cite{Bartlett2017, Golowich2017, Neyshabur2018}; see \cite{Nagarajan2019} for a discussion.
\eit
Transductive bounds have received much less attention than inductive bounds in general. Generalization bounds for the transductive setting have been developed within the PAC-Bayesian framework in \cite{derbeko, begin14}. Our bounds are not PAC-Bayes bounds in the strict sense. Our methods do provide a bound on the expected full sample risk as in the PAC-Bayes case, however this bound is computed using Hilbert coresets which is a novel approach. Also, the bounds developed in \cite{derbeko, begin14} are applicable for a uniform distribution on the full sample and have only been applied to simple hypothesis classes. On the other hand, the bounds developed in this paper are applicable to more general data distributions, such as those that lie in the convex hull of the probability simplex of $K$ probability vectors. Finally, to our knowledge our paper is the first one to develop meaningful and tight upper bounds within a transductive framework for state-of-the-art DNNs. 

\section{Problem Formulation}\label{formulation}

We start with considering an arbitrary input space $\mathcal{X}$ and binary output space $\mathcal{Y} = \{-1,+1\}$. For simplicity, we focus on the binary classification case; however our results are easily generalizable to the multi-class classification case as well. Consistent with the original transductive setting of Vapnik \cite{Vapnik1998}, we consider a \emph{full sample} $Z$ of $N$ examples (input-output pairs) that is assumed to be drawn i.i.d from an unknown distribution $D$ on $\mathcal{X}\times \mathcal{Y}$: 
\be Z:= \{(x_n,y_n)\}_{n=1}^N \ee For simplicity we will refer to the probability distribution as $D \equiv \{p_n\} = \{P(x_n, y_n)\}; n=1,2,..,N$. 
$D$ is assumed to lie in a class of data distributions $\mathcal{D}$. As mentioned previously, we consider the class of distributions where $\mathcal{D}$ is the \emph{convex hull of a probability simplex} $\mathcal{D}_{CH}^{(K)}(\{\hat{p}^{(i)}\})_{i \in [K]}$ defined by the $K$ probability vectors $(\{\hat{p}^{(i)}\}); i \in [K]$. A training dataset $Z^{(S)} =\{(x_i, y_i)_{i=1}^{N_S}\} \in \mathcal{X}\times\mathcal{Y}$, that contains $N_S$ input-output pairs is obtained from the full sample by sampling without replacement. Given $Z^{(S)}$ and the unlabeled set $U_S$  consisting of $N-N_S$ inputs, the goal in transductive learning is to learn a classifier that correctly classifies the unlabeled inputs of the set $U_S$, i.e. minimize the risk of the classifier on $U_S$. 

We consider in this work general loss functions (including unbounded loss functions) $l: \mathcal{Y} \times \mathcal{Y} \rightarrow \mathbf{R}$, and a (discrete or continuous) set $\mathcal{H}$ of hypotheses $h: \Theta \times \mathcal{X} \rightarrow \mathcal{Y}$ where $\theta$ lies in some parameter space $\Theta$. Although the bounds we develop apply to general loss functions, most of the relevant applications of our bounds for classification apply to bounded, and in particular \emph{zero-one}, loss functions. We provide a bound on the \emph{full sample risk} of a hypothesis $h \in \mathcal{H}$, which is given by:
\be\label{sample-risk}
R(h; \theta) = \mathbf{E}_{Z \sim D}\;l(h(\theta, x), y)
\ee
Inspired by the PAC-Bayesian approach, we study an averaging of the full sample risk according to a posterior distribution $\nu$ over $\mathcal{H}$ that is dependent on the data and the training algorithm in general. 
Since $\mathcal{H}$ is parametrized by $\theta \in \Theta$, by a slight abuse of notation we denote the distribution on the parameters $\theta$ by $\nu(\theta)$ as well. So, we aim to provide bounds on the \emph{expected full sample risk} $R(\nu)$:
\be \label{expected-sample-risk}
R(\nu) = \mathbf{E}_{\theta \sim \nu}\,R(h; \theta) = \mathbf{E}_{\theta \sim \nu}\; \mathbf{E}_{Z \sim D} \;l(h(\theta, x), y)   = \mathbf{E}_{\theta \sim \nu} \left(\sum_{n=1}^N \, p_n\, l_n(\theta) \right) : =  \mathbf{E}_{\theta \sim \nu} \, L(\theta)
\ee
where the individual loss for a data point is defined as $l_n(\theta):= l(h(\theta, x_n),y_n)$ and $L(\theta):= \sum_{n=1}^N\,p_n\,l_n(\theta)$ is the total loss function of the full sample. Both $L(\theta)$ and $l_n(\theta), n\in [N]$ are functions of the parameters $\theta$. We aim to develop upper bounds on $R(\nu)$ using the Hilbert coresets approach, which we summarize next.

\subsection{Summary of Hilbert Coresets Approach}\label{coresets-summary}

As described above, we consider a dataset $Z  = \{(x_n,y_n)\}_{n=1}^N \sim D$, parameters $\theta \in \Theta$, and a \emph{cost function} $L: \Theta \rightarrow \mathbf{R}$ that is additively decomposable into a set of functions $L_n(\theta)$, i.e. $L(\theta):= \sum_{n=1}^N\, L_n(\theta)$. By comparing to (\ref{expected-sample-risk}) and the discussion following that, we can make the identification: 
\be
L_n(\theta) := p_n\,l_n (\theta) 
\ee

Now, an \emph{$\epsilon$-coreset} is given by a weighted dataset with nonnegative weights $w_n, n= 1,2...,N$, only a small number $m$ of which are non-zero, such that the weighted cost function $L(w, \theta) := \sum_{n=1}^N\,w_n\,L_n(\theta) $ satisfies: \be \label{coresets-init} |L(w, \theta) - L(\theta)| \leq \epsilon\, L(\theta)\hspace{0.5cm} \forall \theta \in \Theta \ee for $\epsilon > 0$.
If we consider a bounded uniform norm for functions of the form $f: \Theta \rightarrow \mathbf{R}$, which is weighted by the total cost $L(\theta)$ as in
\be \label{uni-norm} || f || := \sup_{\theta \in \Theta} \; \left| \frac{f(\theta)}{L(\theta)} \right|, \ee we can rewrite the problem of finding the best coreset of size $m$ as:
\be \label{coresets-2}
\min_{w}\;|| L(w)-L ||^2 \hspace{0.3cm} \mathrm{s.t.} \sum_{n=1}^N \, \mathbf{I}\,[w_n > 0] \leq m;  \;\; w \geq 0 
\ee
The norm $||f||$ of any function $f: \Theta \rightarrow \mathbf{R}$ gets rid of its dependence on $\theta$, thus there is no dependence on $\theta$ in the norm in (\ref{coresets-2}). The key insight of Campbell \emph{et al} \cite{Campbell2017, Campbell2018} was to approach the problem from a geometrical point of view. In particular, the idea is to construct coresets for cost functions in a \emph{Hilbert space}, i.e. in which the cost functions $L_n(\theta)$ (defined in the beginning of this subsection) are endowed with a norm corresponding to an inner product. The inner product provides a notion of directionality to the various $L_n$ and allows us to use more general norms than the uniform norm in (\ref{uni-norm}) (which we will see in section \ref{coresets-fw}). This in turn provides us with two distinct benefits: \bit  \item{It allows one to choose coreset points intelligently based on the residual cost approximation error}. \item{Theoretical guarantees on approximation quality incorporate the error in the approximation via the (mis)alignment of $L_n$ relative to the total cost $L$.} \eit 

\subsection{Hilbert Coresets via Frank-Wolfe}\label{coresets-fw}

Campbell {\it et al} \cite{Campbell2017, Campbell2018} developed various algorithms for computing Hilbert coresets; here we focus on the Frank-Wolfe algorithm since it produces exponential convergence for moderate values of the coreset size $m$. Furthermore, it can be elegantly tied to risk bounds and computed explicitly in many interesting cases. In this case, the problem in (\ref{coresets-2}) is reformulated by replacing the cardinality constraint on the weights with a polytope constraint on $w$:
\bea \label{FW-init}
& & \min_{w}\;|| L(w)-L ||^2 \\
& s.t.& \sum_{n=1}^N w_n ||L_n|| = \sum_{n=1}^N ||L_n||;  \;\; w \geq 0 \nonumber
\eea One can rewrite the objective function $|| L(w)-L ||^2$ as $J(w) := (w-1)^T\,K\,(w-1)$ with the matrix $\hat{K}$ having elements defined in terms of the inner product, $\hat{K}_{n,m} := \inner {L_n} {L_m}$ (which will be defined shortly). This optimization problem is then solved by applying the Frank-Wolfe algorithm for $m$ iterations to find the coresets $w$\cite{Campbell2017}. The size of the coreset (number of non-zero coreset weights) is smaller than $m$ (typically much smaller), see details in Algorithm 2 of \cite{Campbell2017}. Finally, it can be shown that the Frank-Wolfe algorithm satisfies:
\be \label{fw-bound0}
|| L(w) - L || \leq \frac{\sigma\, \eta\, \bar{\eta}\, \beta }{\sqrt{\eta^2\,(m-1)+\bar{\eta}^2\, \beta^{-2(m-2)}}}, 
%& &\eta^2 := 1 - \frac{||L||^2}{\sigma^2},\; \bar{\eta}^2 := max_{n,m}\,||\frac{L_n(\theta)}{||L_n (\theta)||} - \frac{L_m (\theta)}{||L_m(\theta)||} ||, \; \beta := 1- \frac{r^2}{\sigma^2\,\bar{\eta}^2} \nonumber
\ee where $\sigma := \sum_{n=1}^N\,||L_n||$. The remaining quantities $\{\eta,\,\bar{\eta}, \beta\}$ can be computed from all the data ($\{x_n,y_n\}; n = 1,2..., N$) and the $\nu$-weighted $L2$-norm, see \cite{Campbell2017}. Variants of these quantities (which will be relevant for us later) will be defined precisely in section \ref{coresets-bounds-fw}. In particular $0 < \beta < 1$, and for moderately large $m$, the second term in the denominator dominates giving rise to an exponential dependence on $m$ ($ \sim \beta^m$) .

The above bound on the coreset approximation error for the cost $L$ is valid for any choice of norm. However, for our purposes we will be concerned with the $\nu$-weighted $L2$-norm defined by: 
\be \label{l2-norm} ||L(w)-L||^2_{\nu, 2} := \mathbf{E}_{\theta \sim \nu}\,[L(w,\theta)-L(\theta)]^2, \ee with an induced inner product:
\be \label{innerprod} \langle L_n, L_m \rangle := \mathbf{E}_{\theta \sim \nu} [L_n(\theta)\,L_m(\theta)] = K_{n,m}, \ee
where $K_{n,m}$ was defined below (\ref{FW-init}). In the following, we will see that an upper bound on a variant of $||L-L(w)||$ can be connected to an upper bound on the full sample risk.

\section{Risk Bounds via Frank-Wolfe Hilbert Coresets}\label{coresets-bounds-fw}

Leveraging the Frank-Wolfe Hilbert coreset construction scheme,  we provide an upper bound on the expected full sample risk which is valid for all posteriors $\nu \in \mathcal{H}$. However, the precise bound is dependent on the posterior $\nu$. Furthermore, the bound is independent of the data-distribution over the set of data-distributions $\mathcal{D}$. The result can be stated as follows:

\begin{theorem} \label{theorem1}
Given a full sample $Z$ of size $N$, an unknown distribution on the data belonging to $\mathcal{D}^{(K)}_{CH}(\{\hat{p}^{(i)}\})_{i \in [K]}$, a hypothesis space $\mathcal{H}: \Theta \times \mathcal{X} \rightarrow \mathcal{Y}$, a loss function $l: \mathcal{Y} \times \mathcal{Y} \rightarrow \mathbf{R}$, for any prior distribution $\pi_0$ on $\mathcal{H}$, then the output $\{\tilde{w}, p\}$ obtained after $m$ iterations of applying \emph{Algorithm 1} satisfies:
\bea 
\forall \nu \,\mathrm{on} \,\mathcal{H},\, \forall D \in \mathcal{D}_{CH^{(K)}}, \hspace{0.5cm} |R(\nu)| &\leq&  ||L(\tilde{w})||_{\nu, 2} +  ||L(\tilde{w}) - L||_{\nu, 2}  \label{theorem1-ineqA} \\ 
&\leq&  ||L(\tilde{w})||_{\nu, 2}  + \frac{\hat{\sigma}\,\hat{\eta} \,\hat{\bar{\eta}} \,\hat{\beta}} {\sqrt{\hat{\bar{\eta}}^2\, \hat{\beta}^{-2(m-1)} + \hat{\eta}^2 \,(m-1)}},  \label{theorem1-ineqB} 
\eea
\end{theorem}
where $L(\theta)$ is defined as in (\ref{expected-sample-risk}), $L (\tilde{w}, \theta ) := \sum_{n=1}^N\, \tilde{w}_n\,l_n(\theta); \tilde{w}_n := p_n\,w_n$, and the quantities $\{ \hat{\sigma}, \hat{\eta},\,\hat{\bar{\eta}}, \hat{\beta}\}$ are defined as follows: 
\bea \label{consts}
\hat{\sigma} &:=& \max_{i \in [K]}\,\sigma^{(i)};\;\hat{\eta} := \max_{i \in [K]}\, \eta^{(i)};\;\hat{\beta} := \max_{i \in [K]}\,\beta^{(i)}\nonumber \\
\sigma^{(i)} &:=& \sum_{n=1}^N\,\hat{p}^{(i)}_n\,|| l_n ||; \;
{\eta^{(i)}}^2 = 1 - \frac{||L^{(i)}||^2} {(\sigma^{(i)})^2}; \; {\beta^{(i)}}^2 = 1 - \frac{(r^{(i)})^2}{(\sigma^{(i)})^2 \hat{\bar{\eta}}^2}; \;\; i \in [K]\nonumber \\
\hat{\bar{\eta}}^2 &:=& \max_{n,m} \norm{ \frac{l_n}{|| l_n||} - \frac {l_m} {|| l_m ||} };\;L^{(i)}(\theta):= \sum_{n=1}^N \hat{p}^{(i)}_n l_n(\theta), 
\eea
with $r^{(i)}$ corresponding to the shortest distance from the probability vector $\hat{p}^{(i)}$ ($i^{th}$ vertex of the convex hull $\mathcal{D}_{CH}^{(K)}$) to the relative boundary of the feasible region of $\tilde{w}$. All of the quantities above as well as the quantities needed in \emph{Algorithm 1} can be computed from the full sample ($\{x_n,y_n\}; n = 1,2..., N$), the $\nu$-weighted $L2$-norm (\ref{l2-norm}), and inner product (\ref{innerprod}).

\begin{table}
\centering
\scalebox{1.3}{%
\begin{tabular}{|l|}
\hline
\textbf{Algorithm 1 (FW)}\\
\hline
\bf{Given}: $\{l_n (\theta)\}_{n=1}^N,\, m, \langle . \rangle, \, \mathcal{D}_{CH}^{(K)}(\{\hat{p}^{(i)}\})_{i \in [K]}$ \\
\textbf{for} $i$ in $\{ 1, 2, ..., K \}$ \textbf{do} \\
$\quad p_0 \leftarrow \hat{p}^{(i)} \qquad \qquad \qquad \qquad \qquad \qquad$  \textcolor{blue}{ $\triangleleft$ initialize $p$ to $i^{th}$ vertex in $\mathcal{D}_{CH}^{(K)}(\{\hat{p}^{(i)}\})_{i \in [K]}$}\\
$\quad \forall n \in [N]$, compute $||l_n||$ and $\sigma^{(i)}$ from (\ref{consts}) \quad \textcolor{blue}{$\triangleleft$ compute norms} \\
$\quad f_0 \leftarrow \argmax_{n \in [N]} \langle L^{(i)}, \frac{l_n}{||l_n||} \rangle$  \qquad \qquad \qquad \textcolor{blue}{$\triangleleft$ select vertex in $\tilde{w}$-polytope greedily} \\
$\quad \tilde{w}_0^{(i)} \leftarrow \frac{\sigma^{(i)}} {|| l_{f_0}||} \mathbf{1}_{f_0}$ \qquad \qquad \qquad \qquad \qquad \textcolor{blue}{$\triangleleft$ initialize $\tilde{w}$ (for $p = \hat{p}^{(i)}$) with full weight on $f_0$}\\
\quad \textbf{for} $t$ in $\{ 1, 2, ..., m-1 \}$ \textbf{do} \qquad \qquad \qquad\textcolor{blue}{$\triangleleft$ carry out the following for $m-1$ iterations for each $\hat{p}^{(i)}$}\\
$\quad \quad f_t \leftarrow \argmax_{n \in [N]} \inner  {L^{(i)} - L(\tilde{w}^{(i)}_{t-1})}{\frac{l_n}{||l_n ||}}$  \qquad \textcolor{blue}{$\triangleleft$ find FW index at $t^{th}$ iteration}\\
$\quad \quad \gamma_t \leftarrow \frac {\inner {\frac{\sigma^{(i)}} {|| l_{f_t}||} l_{f_t} - L(\tilde{w}^{(i)}_{t-1})} {L^{(i)} - L(\tilde{w}^{(i)}_{t-1})}} {\norm{\frac{\sigma^{(i)}} {|| l_{f_t}||} l_{f_t} - L(\tilde{w}^{(i)}_{t-1})}^2}$  \qquad \qquad \qquad\textcolor{blue}{$\triangleleft$ closed from line search for step-size $\gamma$ at $t^{th}$ iteration}\\
$\quad \quad \tilde{w}^{(i)}_t \leftarrow (1-\gamma_t ) \tilde{w}^{(i)}_{t-1} + \gamma_t \frac{\sigma^{(i)}} {\tilde{\sigma}_{f_t}} \mathbf{1}_{f_t}$ \qquad \qquad \quad \textcolor{blue}{$\triangleleft$ add or reweight $f_t^{th}$ data point }\\
\quad \textbf{end for}\\
\textbf{end for}\\
\textbf{return} $\{\tilde{w}, p\} = \{\tilde{w}^{(I)}_{m-1}, \hat{p}^{(I)}\} = \argmax_{i \in [K]}\,\tilde{J}(\tilde{w}^{(i)}_{m-1}, \hat{p}^{(i)})$\\
\hline
\end{tabular}}
\end{table}\label{algo1}
The proof of the first inequality in (\ref{theorem1}) can be obtained in a straightforward manner from Jensen's inequality. In particular, applying Jensen's inequality to the convex function $\phi(.):= (.)^2$, one has:
\be \phi(\mathbf{E}_{\theta \sim \nu}(.)) \leq \mathbf{E}_{\theta \sim \nu}\,[\phi(.)]\ee Thus, choosing, $(.)$ as $ L(\tilde{w}, \theta) - L(\theta)$ and $L(\tilde{w}, \theta)$  respectively, one gets:
\bea
| \left[ \mathbf{E}_{\theta \sim \nu}\, (L(\tilde{w}, \theta) - L(\theta)) \right] | &\leq & (\mathbf{E}_{\theta \sim \nu}\, (L(\tilde{w}, \theta) - L(\theta))^2)^{1/2} = || L(\tilde{w}, \theta) - L(\theta) ||_{\nu, 2}, \nonumber \\
|\left[ \mathbf{E}_{\theta \sim \nu}\, (L(\tilde{w}, \theta) \right] |  &\leq & (\mathbf{E}_{\theta \sim \nu}\, (L(\tilde{w}, \theta) )^2)^{1/2} = || L(\tilde{w}, \theta) - L(\theta) ||_{\nu, 2}
 \eea
\bea
\implies | R(\nu) | &=& |\mathbf{E}_{\theta \sim \nu}\, L(\theta)| \leq |\left[ \mathbf{E}_{\theta \sim \nu}\, (L(\tilde{w}, \theta) \right] | + | \left[ \mathbf{E}_{\theta \sim \nu}\, (L(\tilde{w}, \theta) - L(\theta)) \right] | \nonumber \\ &\leq& || L(\tilde{w}, \theta) ||_{\nu, 2} + || L(\tilde{w}, \theta) - L(\theta) ||_{\nu, 2} 
\eea where we have used the definition of the expected full sample risk from (\ref{expected-sample-risk}).

The details of the proof of the second inequality in (\ref{theorem1}) are much more involved and are provided in the appendix. Here we point out some salient features. First we can write
\be \label{objective2} || L(\tilde{w}, \theta) - L(\theta) ||^2_{\nu, 2} = (\tilde{w}-p)^T \tilde{K} (\tilde{w} -p)\,\;\; \mathrm{with}\;\; \tilde{K}_{n,m} := \langle l_n, l_m\rangle, \ee For notational simplicity we denote the RHS above 
as $ \tilde{J}(\tilde{w},p)$, which we interpret as an objective function to optimize. The above problem is similar to the Hilbert coreset problem posed in (\ref{FW-init}) with $J(w)$ in  (\ref{FW-init}) replaced 
by a slightly different quantity $\tilde{J}(\tilde{w}, p)$. In contrast to (\ref{FW-init}), the objective function $\tilde{J}(\tilde{w}, p)$ in (\ref{objective2}) depends on two vectors $\{\tilde{w}, p\}$. Since the goal is to provide a bound for all $D \in \mathcal{D}_{CH}^{(K)}$, one can consider the following optimization problem by adapting the problem in (\ref{FW-init}) to this case:
\bea \label{wp-constraints}
& & \max_{D = \{p_n\} \in \mathcal{D}_{CH}^{(K)}}\, \min_{\tilde{w}}\,(\tilde{w}-p)^T \tilde{K} (\tilde{w}-p) \\
&s.t. & \sum_{n=1}^N \tilde{w}_n ||l_n|| = \sum_{n=1}^N p_n ||l_n||, \, \tilde{w} \geq 0 \nonumber \\
&s.t. & p \in \mathcal{D}_{CH}^{(K)}(\{\hat{p}^{(i)}\})_{i \in [K]},\; \sum_{n=1}^N p_n = 1, p \geq 0. \nonumber
\eea Please refer to the appendix for the full proof. An important point to note about the result in Theorem \ref{theorem1} is that the upper bound on the full sample risk denoted by the RHS in (\ref{theorem1-ineqB}) approximately goes like $\hat{\beta}^m$ even for moderate number of Frank-Wolfe iterations $m$.

Theorem (\ref{theorem1}) is valid for the $\nu$-weighted $L2$-norm. However, this norm and the associated inner product involves computing an expectation over the parameters $\theta$ (see (\ref{l2-norm}) and (\ref{innerprod})) which does not admit a closed form evaluation in general. The expectation can be estimated in an unbiased manner by the method of random projections, as in \cite{Campbell2017, Campbell2018, Rahimi2008}. In particular, we construct a random projection of each vector $l_n(\theta); n \in [N]$  into a finite ($J$) dimensional vector space using $J$ samples $\theta \sim \nu(\theta)$:
\be \label{rand-proj}u_n = \sqrt{\frac{1}{J}}[l_n(\theta_1), l_n(\theta_2), ..., l_n(\theta_J)]^T, n \in [N],\ee
which serves as finite dimensional approximation of $L_n(\theta)$. In particular, for all $n, m \in [N]$, we have: \be \langle l_n, l_m\rangle \approx u_n^T\,u_m.\ee Leveraging similar techniques as in \cite{Campbell2017}, we can construct \emph{Algorithm 2}, and arrive at the following result.
\begin{table}
\centering
\scalebox{1.3}{%
\begin{tabular}{|l|}
\hline
\textbf{Algorithm 2 (RP-FW)}\\
\hline
\bf{Given}: $\{l_n(\theta)\}_{n=1}^N, \nu (\theta), J, m, \mathcal{D}_{CH}^{(K)}(\{\hat{p}^{(i)}\})_{i \in [K]}$ \\
$\cdot$ $\nu (\theta)$: $\{\theta_j\}^J_{j=1} \sim \nu (\theta) $ \qquad \qquad \qquad \qquad  \qquad \qquad  \textcolor{blue}{$\triangleleft$ take $J$ samples from posterior $\nu(\theta)$}\\
$\cdot$  $u_n = \sqrt{\frac{1}{J}} [l_n(\theta_1),...,l_n(\theta_J)]^T$  $\forall n \in [N]$ \quad  \qquad \qquad \textcolor{blue}{$\triangleleft$ construct random projection vector for $l_n(\theta)$ using (\ref{rand-proj})}\\
$\cdot$ \textbf{return} $\{\tilde{w},p\} = FW[u_n, m, (.)^T (.), \mathcal{D}^{(K)}_{CH} (\{\hat{p}^{(i)}\})_{i \in [K]}]$ \; \textcolor{blue}{$\triangleleft$ construct coreset with random projection vectors $u_n$}\\ 
\qquad \qquad \qquad \qquad \qquad \qquad \qquad \qquad \qquad \qquad  \textcolor{blue}{by applying \emph{Algorithm 1} with inner product $\langle u_n, u_m\rangle = u_n^T\,u_m$}\\
\hline
\end{tabular}}
\end{table}\label{algo2}

\begin{theorem} \label{theorem2}
Given the general setup of theorem \ref{theorem1}, for loss functions $l$ such that $l_n (\mu)\,l_m(\mu)$ (with $\mu \sim \nu(\theta) $) is sub-Gaussian with parameter $\xi^2$,  for $0 < \delta < 1$, for $\mathcal{D} = \mathcal{D}^{(K)}_{CH} (\{\hat{p}^{(i)}\})_{i \in [K]}$, the output of \emph{Algorithm 2} satisfies with probability $> 1- \delta$:
\bea \label{ineq2}
|R(\nu)| &\leq& ||L(\tilde{w}) ||_{\nu, 2} + || L(\tilde{w}) - L ||_{\nu, 2} \nonumber \\
&\leq& \left(|| u(\tilde{w})||^2_{\nu, 2} +  || \tilde{w} ||_1^2 \,\sqrt{\frac{2\xi^2}{J}\log{\frac{2N^2}{\delta}}}\right)^{1/2} + \left(|| u(\tilde{w}) - u ||^2_{\nu, 2} + || \tilde{w} - p||_1^2 \,\sqrt{\frac{2\xi^2}{J}\log{\frac{2N^2}{\delta}}}\right)^{1/2},  \label{ineq2-A}\\ 
&\leq& \left(|| u(\tilde{w})||^2_{\nu, 2} +  || \tilde{w} ||_1^2 \,\sqrt{\frac{2\xi^2}{J}\log{\frac{2N^2}{\delta}}}\right)^{1/2} + \left(  \frac{\hat{\sigma}\,\hat{\eta} \,\hat{\bar{\eta}} \,\hat{\beta}}{\sqrt{\hat{\bar{\eta}}^2\, \hat{\beta}^{-2(m-1)} + \hat{\eta}^2 \,(m-1)}} + || \tilde{w} - p||_1^2 \,\sqrt{\frac{2\xi^2}{J}\log{\frac{2N^2}{\delta}}}\right)^{1/2},  \label{ineq2-B}
\eea 
The quantities $\{ \hat{\sigma}, \hat{\eta},\,\hat{\bar{\eta}}, \hat{\beta}\}$ are defined as before. The sub-Gaussianity parameter $\xi$ can be estimated from the loss functions. For the $0-1$ loss function, a simple bound is $\xi \leq 1/2$. Also note that $\{\tilde{w}, p\} = \{\tilde{w}_{m-1}^{(I)}, \hat{p}^{(I)}\}$ where $I \in [K]$ is the index that maximizes $\tilde{J}(\tilde{w}^{(i)}_{M-1}, \hat{p}^{(i)})$ (see \emph{Algorithm 1}). The details of the proof can be found in the appendix.
\end{theorem}   

\section{Experiments and Instantiation of Bounds}\label{expts}

In this section, we present our experimental setup for computing Hilbert coreset bounds for three datasets: i) Binary MNIST, and ii) CIFAR-100. MNIST is a very well known image classification dataset and is usually treated as a multiclass problem. We follow the results of \cite{Dziugaite2017} and create a binary classification dataset out of it by labeling the digits 0-4 with class `-1' and digits 5-9 with class `1'. The CIFAR-100 datasets is also a well known dataset with hundred different classes of images. 

The distribution on the data in each case is assumed to be a uniform distribution for simplicity, with $p_n = \frac{1}{N}$ for a full sample of size $N$. In other words, the class of data distributions, e.g. the convex hull $\mathcal{D}_{CH}^{(K)}$ is simply a point. This can be extended in a straightforward manner to cases where  $\mathcal{D}_{CH}^{(K)}$ using the methods described in the appendix. Finally, the loss function $l$ used for the computation of bounds is the standard 0-1 loss used in classification. Further experimental details for each of these datasets are provided below. 

\begin{subsection}{Binary MNIST}\label{mnist}
The MNIST dataset is obtained from Keras. This dataset is split into the training set (60000 images) and test set (10000 images).
We convert this into a binary classification dataset as discussed above. The images are preprocessed by scaling them to float values in the range $[0,1]$.
Keras is used to construct a simple neural network binary classifier, similar to that in \cite{Dziugaite2017}.
In particular, the hypothesis class $\mathcal{H}$ for this dataset is taken to be a feed-forward neural network which has three hidden layers each with 600 nodes and a ReLu activation function.
The output layer of the network uses a sigmoid activation function. The weights at each node are initialized according to a truncated normal distribution with mean $0$ and standard deviation $0.04$.
The biases of the first hidden layer are 0.1, whereas the biases of the second hidden layer and output layer are initialized to 0. 

The model is trained on different random subsets of the MNIST dataset using the cross-entropy loss (however the bounds will be provided with respect to the 0-1 loss).The first checkpoint uses 500 data points from the training set. The second checkpoint used 3000 data points, including the 500 from the first checkpoint.The third checkpoint uses 10000 data points, including the 3000 already used, and the final checkpoint uses 30000 data points. The models were trained for 120 epochs for each checkpoint.
\end{subsection}

%\begin{subsection}{CIFAR-10}\label{cifar10}

%For the CIFAR-10 dataset, we construct a classifier similar to that constructed in \cite{Banerjee2020}. 
%In particular, the hypothesis class $\mathcal{H}$ in this case is a fully-connected neural network with four layers each with 256 nodes and a ReLu activation function.
%For the final layer we use a softmax activation function.
%Once again, we initialize the weights with a truncated normal distribution with mean 0 and standard deviation 0.04.
%The biases are initialized for 0.1 on the first layer and 0 on the other layers.

%The CIFAR-10 dataset within Keras is split evenly across the 10 categories and with 50000 images in the training set, and 10000 images in the test set.
%For the first checkpoint, we use 100 images, with 10 chosen randomly from each category as seeds. For the second checkpoint, we use an additional 40 images from each category, for a total of 500.
%For the third checkpoint, we choose an additional 500 images from the dataset (independent of labels) for a total of 1000. For the fourth checkpoint we choose another 4000 images for a total of 5000.
%Finally, we choose an additional 5000 images for a total of 10000. These checkpoints are very similar to that used for training in \cite{Banerjee2020}. The parameters for the optimization are the same as for the binary %MNIST training. For each checkpoint, the model is trained for 120 epochs, at which point the accuracy of the model on the test set is saturated. 
%\end{subsection}

\begin{subsection}{CIFAR-100}\label{cifar100}
For this dataset, we use a state-of-the art neural network architecture as the hypothesis class and a training algorithm is described in Xie et. al. \cite{xie2020muscle}.
The reason for doing this is the following. As a performer on DARPA Learning with Less Labeling (LwLL) program, the goal of our team was to develop novel theoretical bounds 
for other performers who are working on the model-building and training of new algorithms for efficient learning of classifiers with fewer labels. As part of this exercise, we approached the authors in \cite{xie2020muscle}
(who are part of the Information Sciences Institute (ISI) team) for the training models developed by them and developed theoretical bounds for those models (the associated hypothesis classes $\mathcal{H}$). As we will show below, we obtained tight non-vacuous upper bounds for such hypothesis classes as well, indicating the power and diversity of the proposed approach. 

We summarize the training algorithm used in \cite{xie2020muscle} for completeness. \cite{xie2020muscle} uses mutual-information based unsupervised $\&$ semi-supervised Concurrent learning (MUSCLE), a hybrid learning approach that uses mutual information to combine both unsupervised and semisupervised learning. MUSCLE can be used as a standalone training scheme for neural networks, and can also be
incorporated into other learning approaches. It is shown that the proposed hybrid model outperforms state of the art for the CIFAR-100 dataset as well as others. The basic idea of MUSCLE is to combine combine 
the loss from (semi)-supervised learning methods including real and pseudo labels with a loss term arising from the mutual information between two different transformed versions of the input sample ($x_i$). Furthermore, a consistency loss that measures the consistency of the prediction on an input sample and an augmented sample, can be added as a regularization term. Please refer to  \cite{xie2020muscle} and references therein for further details about the training algorithm and neural network models (hypothesis classes). 

The model is trained with different random subsets of the data, at the 500, 2500, 4000 and 10000 checkpoints.

\end{subsection}

\subsection{Choice of Prior and Posterior Distributions} \label{prior-posterior}

As explained in section \ref{intro}, the theoretical bounds we develop are stochastic bounds. So, one has to choose an appropriate prior distribution (any distribution independent of the data) and a posterior distribution (that depends on the training data and the training algorithm in general). In our analysis, we choose the prior distribution to be $\pi_0 = N(0, \sigma^2\cdot \mathbf{I})$. Inspired by the analysis of Banerjee {\it et al} \cite{Banerjee2020}, we use a training data-dependent posterior $\nu_S = N(\theta^{\dag}, \Sigma^S_{\theta^{\dag}})$, where $\theta^{\dag}$ denotes the parameters \emph{learned} by training the models (in the previous subsections) on the training set $S$, and the covariance $\Sigma_{\theta^{\dag}}$ is defined as follows: 
\bea \label{hessian}
\Sigma^S_{\theta^{\dag}} &=& diag(\sigma_j^2); \; \sigma_j^{-2} := \max\,[H^S_{j,j}(\theta^{\dag}), \sigma^{-2}], \;\;j= 1,2,..,dim(\Theta) \nonumber\\
\mathrm{where}\;\; H^S(\theta^{\dag}) &:=& \frac{1}{N_S} \sum_{i=1}^{N_S} \, \nabla^2\,l(y_i, \phi(\theta^{\dag}, x_i))\;\;\mathrm{(Diagonal\; Hessian)}
\eea
The anisotropy in the posterior $\nu_S$ can be understood as follows. For directions $j$ in parameter space where $H^S_{j,j}(\theta^{\dag})$ is \emph{larger} than $\sigma^{-2}$, the posterior variance is determined by the Hessian, and will be small enough such that $\theta_j$ will not deviate far from the learned parameters $\theta^{\dag}_j$ when sampling from the posterior $\nu_S$. On the other hand, for directions $j$ in parameter space where $H^S_{j,j}(\theta^{\dag})$ is \emph{smaller} than $\sigma^{-2}$, the posterior variance $\sigma_j^2$ is the same as that of the prior $\sigma^2$. Thus, the posterior standard deviation is upper bounded by the prior standard deviation and different weights and biases are assumed to be independent. One can optimize over the prior standard deviation $\sigma$ to get a good bound. 

\subsection{Bound Calculation Workflow}\label{workflow}

Given the data points $(x_n, y_n), n=1,2,..,N_S$ of the training set checkpoint $S$ of a particular dataset, a trained model (with the learned parameters $\theta^{\dag}$) for that checkpoint, and a prior distribution $\pi_0= N(0, \sigma^2\cdot \mathbf{I})$ with an appropriately chosen $\sigma$, the workflow for computing the coreset bounds corresponding to the training set $S$ is as follows:
\begin{itemize}
 \item We compute the diagonal Hessian of the (cross-entropy) loss for the trained model at the trained parameters $H^S_{j,j}(\theta^{\dag})$ for the training set $S$ via the methods described in \cite{Sagun2017}. Then we compute the posterior variance($\sigma_j^2$) as described in section \ref{prior-posterior}. This gives the data dependent posterior $\nu_S$ for the training checkpoint $S$.
 \item Given $\nu_S$, we compute an approximation to the inner product $\langle \cdot \rangle$ by the method of \emph{random features projection} \cite{Rahimi2008, Campbell2017} as in \emph{Algorithm 2} which converge to the true inner products {\it with high probability} as the number of samples from the posterior ($J$) increases (see Theorem \ref{theorem2}). 
 \item From the inner products computed with the posterior for the training set $S$, and using the full dataset of size $N$, we compute all the constants $\{ \hat{\sigma}, \hat{\eta},\,\hat{\bar{\eta}}, \hat{\beta}\}$ in Theorems \ref{theorem1} and \ref{theorem2} and then implement \emph{Algorithm 1} with number of iterations $m$ equal to training set checkpoint size $N_S$ to find the coreset weights $\tilde{w}^{(I)}_{m-1}$ and the appropriate probability vector $\hat{p}^{(I)}$.
 \item Finally, using the coreset weights and the inner product $(.)^T\,(.)$, we compute \emph{coreset upper bound-I} (first term in the RHS of (\ref{ineq2-A})) as well as \emph{coreset upper bound-II} (RHS of (\ref{ineq2-B})).
\end{itemize}

The above workflow was carried out for all different training checkpoints (of increasing sizes) for both datasets mentioned in sections \ref{mnist} and \ref{cifar100}.

\subsection{Experimental Results}\label{results}

\begin{figure}[h]
\begin{tabular}{ll}
\includegraphics[scale=0.6]{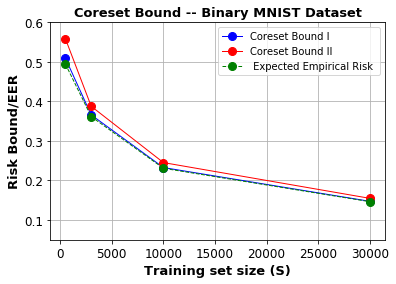}
&
\includegraphics[scale=0.6]{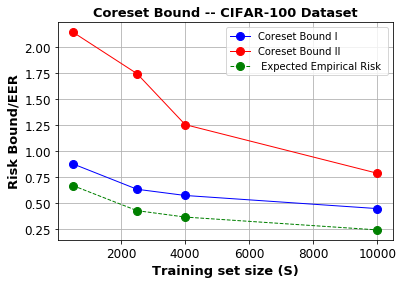}
\end{tabular}
\caption{\textbf{Left}: \footnotesize{\emph{Coreset Upper Bound - I}  (RHS of (\ref{ineq2-A}), shown in blue) and \emph{Coreset Upper Bound -II} (RHS of (\ref{ineq2-B}), shown in red) on the expected full sample risk for the three-layer ReLU DNN model for four training checkpoints of the Binary MNIST dataset with prior standard deviation $\sigma = 10^{-4}$, based on Theorem \ref{theorem2}. The expected empirical risk (EER) of the trained model on the full sample is shown in green.}
\textbf{Right}: \footnotesize{\emph{Coreset Upper Bound - I}  (RHS of (\ref{ineq2-A}), shown in blue) and \emph{Coreset Upper Bound -II} (RHS of (\ref{ineq2-B}), shown in red) on the expected full sample risk for the Xie \emph{et al} \cite{xie2020muscle} model for four training checkpoints of the CIFAR-100 dataset with prior standard deviation $\sigma = 10^{-4}$, based on Theorem \ref{theorem2}. The expected empirical risk (EER) of the trained model on the full sample is shown in green.}
}\label{plots}
\end{figure}
For our numerical results, we choose $J$, the number of samples from the posterior (see section \ref{workflow}), as 7000. We experimented with other values of $J$ and found that the numerical results for the bounds were stable with $J \gtrsim 7000$. We also experimented with the choice of the prior standard deviation $\sigma$. As discussed in section \ref{prior-posterior}, the posterior standard deviation is upper bounded by $\sigma$. So, we would like to choose a low-value of $\sigma$ to come up with an \emph{effectively deterministic} bound. After some experimentation, we chose a value of $\sigma = 10^{-4}$ that seemed to provide good bounds which were effectively deterministic in the sense that the expected empirical risk (EER) is very close to the empirical risk. For each dataset, we use the number of Frank-Wolfe iterations $m$ (in \emph{Algorithm 1}) equal to the size of the training set checkpoint $S$. We find that the size of the coreset (the number of non-zero weights) is \emph{much smaller} than the number of Frank-Wolfe iterations or the size of the training set $S$. For example, for the CIFAR-100 dataset, the coreset size is only $\sim150$, around two orders of magnitude smaller than the size of the largest training set checkpoint $S = 10000$. This is also consistent with results in \cite{Campbell2017, Campbell2018}. 

The left plot in Figure \ref{plots} shows that the bounds for the MNIST model are very tight. For example, the empirical expected full sample risk at the first checkpoint with 500 samples (0.49) is quite close to the \emph{coreset upper bound-I} (0.51) and the \emph{coreset upper bound-II} (0.56). We get tight non-vacuous bounds for the DNN hypothesis class used in section \ref{mnist} for all training checkpoints. Also, the coreset upper bounds obtained for the training checkpoint with 30000 samples ($0.145\, \& \,0.154$) are better than that obtained in \cite{Dziugaite2017} for 55000 samples (0.207). As discussed earlier, we use the same model and training procedure as \cite{Dziugaite2017} to allow a fair comparison.

The right plot in Figure \ref{plots} shows the results for the CIFAR-100 model with the state-of-the-art neural network architecture and training algorithm as in \cite{xie2020muscle}. For this model, the \emph{coreset upper bound-I} is non-vacuous at all training checkpoints, while the  \emph{coreset upper bound-II} is non-vacuous for training set sizes larger than around 8000. Thus we are able to compute meaningful and effectively deterministic upper bounds for the full sample risk for complex datasets (like CIFAR-100) and state-of-the-art hypothesis classes and training algorithms, as in \cite{xie2020muscle} in section \ref{cifar100}. 

The constants in the bounds in Theorems \ref{theorem1} and \ref{theorem2} depend on the posterior distribution $\nu$. As discussed in section \ref{workflow}, the posterior distribution $\nu_S$ is computed for each training checkpoint $S$. Therefore, the constants in the bounds plotted in Figure \ref{plots} also change with each training checkpoint $S$. An important point to note is that the bounds for both datasets/models decrease significantly with the size of the training set $S$. This is in contrast to some other recent bounds for DNNs that may even increase with the size of the training set \cite{Bartlett2017, Golowich2017, Neyshabur2018}, see discussion in \cite{Nagarajan2019}. 

Finally, as pointed out in the beginning of section \ref{expts}, we have assumed a uniform distribution on the data for simplicity implying that the class of data distributions $\mathcal{D}$ is trivial. However, it is straightforward to extend these results to non-trivial classes like $\mathcal{D}_{CH}^{(K)}$, using the methods described in the appendix. The expected empirical full sample risk will still be upper bounded by the bounds computed for these data distribution classes as long as the uniform distribution belongs to $\mathcal{D}_{CH}^{(K)}$.

\begin{section}{Conclusions}

In this paper, we have developed the formalism for computing upper bounds on the \emph{expected full sample risk} for supervised classification tasks within a transductive setting using the method of Hilbert coresets. We compute explicit numerical bounds for complex datasets and complex hypothesis classes (such as state-of-the-art DNNs) and compare these to the performance of modern training algorithms. The bounds obtained are tight and meaningful. The bounds we develop are stochastic, 
however they are effectively deterministic for appropriate choice of priors. The bounds are \emph{non-uniform} in the hypothesis space, since the precise numerical values of the bound depend on the learned weights of the training algorithm ($\theta_{\star}$). Finally, the bounds we develop decrease significantly with the size of the training set ($S$)  as can be seen from section \ref{results}. 

There are a couple of avenues for future research. First, it should be possible to derive generalization bounds for other classes of data distributions $\mathcal{D}$, such as in which the probability vector for the full sample lies inside an ellipsoid. Second, as mentioned in section \ref{results}, the coreset size is much smaller than the size of the training set $S$ for a given training checkpoint. This suggests that it might be possible to develop a model-free (agnostic) active learning approach leveraging the coresets framework.  For example, given a prior and an initial (presumably small) set of labeled data, one can use the coresets approach to construct an initial posterior on the hypothesis space. Using this posterior, one can make predictions on the unlabeled data and compute a \emph{pseudo risk} for the entire dataset. Then one could use this \emph{pseudo-risk} to construct the next set of coreset weights and update the coreset weighted risk. At each step, one should also update the posterior appropriately. The goal would be to show that this procedure converges and provides an upper bound on the expected true risk. The hope is that this should provide much tighter generalization and sample complexity bounds. 

\end{section}

\section{Acknowledgement}
The material in this paper is based upon work supported by the United States Air Force and DARPA under Contract No. FA8750-19-C-0205 for the DARPA I20 \emph{Learning with Less Labeling (LwLL)} program. S.D., P.K., and R.K.P are grateful to Prof. T. Broderick for providing inspiration to the ideas in this paper as well as providing feedback on the manuscript.  The authors would like to particularly thank Prof. V. Chandrasekaran and Prof. A. Willsky for insightful discussions and comments on the manuscript.

\appendix

\subsection{Proof of Theorem \ref{theorem1} for  $\mathcal{D} = \mathcal{D}_{CH}^{(K)}(\hat{p}^{(I)}), I=1,2..,K$ } \label{polytope}

In this case,  we define $\mathcal{D}$ as the convex hull of $K$ points in the probability simplex characterized by $N$ points.

\begin{lemma} \label{polytope-max}
Let  $\mathcal{D}_{CH}^{(K)}(\hat{p}^{(I)})$ be the convex hull of $\{\hat{p}^{(I)}; 1 \leq I \leq K\}$, where each $\hat{p}^{(I}$ is 
an $N$-dimensional probability vector with $\sum_{n=1}^N\, \hat{p}^{(I)}_n = 1; \hat{p}^{(I)}_n > 0\; \forall n \in [N]$.
Then the optimal solution to $\max_{p \in \mathcal{D}_{CH}^{(K)}}\, \tilde{J}(\tilde{w}, p)$ for a fixed $\tilde{w}$ is $\hat{p}^{(k)}$ for some $k$ in $\{1,2,..,K\}$. 
\end{lemma}

\begin{proof}
Let $p_0$ be any point which is \emph{not} a vertex in the convex hull and fix $\tilde{w}$.
Because $\tilde{J}$ is quadratic, we can represent the function as 

$$\tilde{J} (\tilde{w},p) = \tilde{J} (\tilde{w}, p_0) + \nabla \tilde{J} (\tilde{w}, p_0)^T (p - p_0) + (p - p_0)^T \tilde{K} (p - p_0)$$

Since $p_0$ is not an extreme point, there is a vector $d$ such that both $p_0 + d$ and $p_0 - d$ are in the convex hull.
Clearly, either $\nabla \tilde{J} (\tilde{w}, p_0)^T d \geq 0$ or $-\nabla \tilde{J} (\tilde{w}, p_0)^T d \geq 0$.
So without loss of generality, we can claim there is a vector $d$ such that $\nabla \tilde{J} (\tilde{W}, p_0)^T d \geq 0$.
Furthermore, because $\tilde{J}$ is quadratic, $\tilde{K}$ is positive definite.
So we conclude $p_0$ is suboptimal because

$$\tilde{J} (\tilde{w}, p_0) + \nabla \tilde{J} (\tilde{w}, p_0)^T (p - p_0) + (p - p_0)^T \tilde{K} (p - p_0) > \tilde{J} (\tilde{w}, p_0)$$
\end{proof}

We now provide proofs pertaining to \emph{Algorithm 1} for $\mathcal{D} = \mathcal{D}_{CH}^{(K)}$. 

\begin{lemma}
\emph{Algorithm 1} is equivalent to the Frank-Wolfe coreset construction of \cite{Campbell2017} for $K=1$.
\end{lemma}

\begin{proof}
For $K=1$, the outer maximization is trivial. Since all elements of $\hat{p}^{(1)}$ are non-zero, so we can convert back to $w$. Notice that the objective then becomes

\be \tilde{J}(\tilde{w},p) = (\tilde{w} - \hat{p}^{(1)})^T \tilde{K} (\tilde{w} - \hat{p}^{(1)}) = (w-1)^T \hat{K} (w-1) \ee after writing $w = \tilde{w}\,\hat{p}^{(1)}$ and identifying $\hat{p}^{(1)}_n\, l_n$ with $L_n$. 
Thus, we have:
\be \hat{p}^{(1)}_n ||l_n || = ||L_n ||, \ee 
and the solution to the original Frank-Wolfe algorithm in \cite{Campbell2017} is also a solution to \emph{Algorithm 1}. 
\end{proof}

\textit{Proof of Theorem \ref{theorem1} for $\mathcal{D}_{CH}^{(K)}; K > 1$}: 

Let us first define a number of constants that would be helpful later:
\bea 
\tilde{\sigma}_n &=& || l_n ||; \;
\sigma^{(i)} = \sum_{n=1}^N \hat{p}^{(i)}_n \tilde{\sigma}_n; \;
L^{(i)} = \sum_{n=1}^N \hat{p}^{(i)}_n l_n \nonumber \\
\hat{\bar{\eta}}^2 &=& \max_{n,m} \norm{ \frac{l_n}{|| l_n||} - \frac {l_m} {|| l_m ||} };\;
{\eta^{(i)}}^2 = 1 - \frac{||L^{(i)}||^2} {(\sigma^{(i)})^2}; \; {\beta^{(i)}}^2 = 1 - \frac{(r^{(i)})^2}{(\sigma^{(i)})^2 \bar{\eta}^2}, 
\eea
where $r^{(i)}$ is the shortest distance from the solution $\hat{p}^{(i)}$ to the relative boundary of the feasible region for $\tilde{w}$. The first inequality in Theorem \ref{theorem1} was proven in the main text. 
We now prove the second inequality (see (\ref{theorem1-ineqB})) in Theorem \ref{theorem1}.

\begin{theorem} \label{coreset-alg1}
Let $\hat{\sigma} = \max_{i \in [K]} \sigma^{(i)}$, $\hat{\eta} = \max_{i \in [K]} \eta^{(i)}$ , and $\hat{\beta} = \max_{i \in [K]} \beta^{(i)}$.
Then, with $L = \sum_{n=1}^N\, p_n l_n$ and $L(\tilde{w}) = \sum_{n=1}^N\,\tilde{w}_n\,l_n = \sum_{n-1}^N\,w_n p_n l_n $, the coreset constructed with \emph{Algorithm 1} gives rise to a Hilbert coreset of size at most $m$ with 
\be || L - L(\tilde{w}) || \leq \frac{\hat{\sigma}\hat{\eta}\hat{ \bar{\eta}} \hat{\beta}} {\sqrt{\hat{\bar{\eta}}^2 \hat{\beta}^{-2(m-1)} + \hat{\eta}^2 (m-1)}}\ee
\end{theorem}

\begin{proof}
In order to prove theorem (\ref{coreset-alg1}), we first borrow two lemmas from \cite{Campbell2017} which will be useful shortly. 

\begin{lemma}\label{interior}
If $\hat{p}^{(i)}$ has no non-zero entries, then $L^{(i)} \equiv \sum_{n=1}^N\,\hat{p}^{(i)}_n\,l_n$ is in the relative interior of the convex hull: $Conv\{ \frac {\sigma^{(i)}}{\tilde{\sigma}_n} l_n\}_{n=1}^N$
\end{lemma}
The proof the above is a simple consequence of Lemma A.5 in \cite{Campbell2017} and by the following identification:
\be  \frac{\sigma^{(i)}} {\tilde{\sigma}_n} l_n = \frac{\sigma^{(i)}} {\hat{p}^{(i)}_n || l_n||} \hat{p}^{(i)}_n l_n = \frac{\sigma^{(i)}} {||L_n||} L^{(i)}_n \ee

\begin{lemma}\label{sequence}
Let $\{x_n \}$ be a sequence of real numbers and $0 \leq \alpha \leq 1$ which satisfy $x_{k+1} \leq \alpha x_k (1-x_k)$ for all $k$.
Then for all $n$, \be x_n \leq \frac{x_0} {\alpha^{-n} + nx_0}\ee
\end{lemma}
This is listed as lemma A.6 in \cite{Campbell2017}. 

From Lemma \ref{polytope-max}, we know the outer maximization is solved by one of the vertices of the convex hull $\mathcal{D}_{CH}^{(K)}$. Therefore, following \emph{Algorithm 1} in Table 1, we first select a vertex $\hat{p}^{(i)}$ from the $K$ vertices in $\mathcal{D}_{CH}^{(K)}$ at random. Then for that $\hat{p}^{(i)}$, we use the Frank-Wolfe algorithm to minimize $\tilde{w}$ subject to a convex polytope constraint on $\tilde{w}$ as in (\ref{wp-constraints}). This gives a value of $\tilde{w}_t^{(i)}$ after $t = m-1$ iterations corresponding to $\hat{p}^{(i)}$. The same procedure is then carried out for all $K$ vertices of $\mathcal{D}_{CH}^{(K)}$, and we choose the index $(I)$ that maximizes the objective $\tilde{J}(\tilde{w}_t^{(i)}, \hat{p}^{(i)}); i=1,2,..,K$. 

We now show that the $\tilde{w}_t^{(I)}$ obtained for the index $(I)$ satisfies Theorem \ref{theorem1}. For simplicity, we omit the index $(I)$ from $\tilde{w}_t^{(I)}$ and just use $\tilde{w}_t$.
We first initialize the Frank-Wolfe algorithm by choosing an appropriate $\tilde{w}_0$ as follows:
\bea
f_0 &=& \argmax_{n \in [N]} \, \langle \frac{l_n}{|| l_n||}, \sum_{n=1}^N\,{\hat{p}_n^{(I)}}\,l_n\rangle \nonumber\\
\tilde{w}_0 &=& \frac{\sigma^{(I_0)}}{|| l_{f_0}||}\mathbf{1}_{f_0} = \frac{\sigma^{(I)}}{|| \tilde{\sigma}_{f_0}||}\mathbf{1}_{f_0}
\eea
We now show that $\tilde{J}(\tilde{w}_0,\hat{p}^{(I)}) \leq \sigma^2 \eta^2$. Notice that for any $\xi \in \mathbb{R}^N$ such that $\sum_{n=1}^N \xi_n = 1$,
\bea
\frac{\tilde{J}(\tilde{w}_0,\hat{p}^{(I)})} {{\sigma^{(I)}}^2} &=& 1 - 2< \frac {l_{f_0}} {\tilde{\sigma_{f_0}}},\frac{L^{(I)}} {\sigma^{(I)}}> + \frac{||L^{(I)}||} {{\sigma^{(I)}}^2} \nonumber\\
&\leq& 1 - 2 \sum_{n=1}^N \xi_n <\frac {l_n} {\sigma_n}, \frac {L^{(I)}} {\sigma^{(I)}}> + \frac {||L^{(I)}||} {{\sigma^{(I)}}^2}, 
\eea
 since $l_{f_0}$ maximizes the inner product above. Choosing $\xi_n = \frac {\hat{p}^{(I)} \tilde{\sigma}_n} {\sigma^{(I)}}$ yields 

\bea
\frac {\tilde{J} (\tilde{w}_0, \hat{p}^{(I)})} {{\sigma^{(I)}}^2} &\leq&  1- 2 \sum_{n=1}^N <\frac {\hat{p}_n^{(I)} l_n} {\sigma^{(I)}}, \frac {L^{(I)}} {\sigma^{(I)}}> + \frac {||L^{(I)}||} {{\sigma^{(I)}}^2} \nonumber \\
&=& 1 - \frac {||L^{(I)}||} {{\sigma^{(I)}}^2} := {\eta^{(I)}}^2
\eea Thus, $\tilde{J} (\tilde{w}_0,\hat{p}^{(I)}) \leq {\sigma^{(I)}}^2 {\eta^{(I)}}^2 \leq \hat{\sigma}^2 \hat{\eta}^2$.

Next, we will look at the sequence of values, $\tilde{J}(\tilde{w}_t, \hat{p}^{(I)})$. The goal is to use Lemma \ref{sequence} on this sequence to get the desired bound.
The Frank-Wolfe algorithm uses the update $\tilde{w}_{t+1} = \tilde{w}_t + \gamma^{(I)}_t\,d_{(I)}$, where 
\bea
 d_{(I)} &=& \argmin_{s} [\nabla J(\tilde{w}_t, \hat{p}^{(I)})^T  s] - \tilde{w}_t \nonumber \\
\mathrm{or}, \; d_{(I)} &=& \argmin_s [2(\tilde{w}_t-\hat{p}^{(I)})^T \, \tilde{K} \, s]- \tilde{w}_t=  \argmin_{s}[ 2(\tilde{w}_t-\hat{p}^{(I)})^T \, \tilde{K} \, \frac{\sigma^{(I)}}{\tilde{\sigma}_n}\,\mathbf{1}_n] -\tilde{w}_t,
\eea
where we have used the fact that the minimizer of a linear function in $s$ over a convex polytope will lie at one of the vertices $\frac{\sigma^{(I)}}{\tilde{\sigma}_n}\,\mathbf{1}_n; \;n\in [N]$.
Then it is straightforward to show that:
\be d_{(I)} = \frac{\sigma^{(I)}}{ \tilde{\sigma}_{f_t}}\mathbf{1}_{f_t} - \tilde{w}_t; \;\;f_t = \argmax_{n \in [N]} \langle L - L(\tilde{w}_t), \frac{l_n}{\tilde{\sigma}_n}\rangle  \ee
This gives:
\be \label{jt}
\tilde{J}(\tilde{w}_{t+1},\hat{p}^{(I)})=\tilde{J}(\tilde{w}_t,\hat{p}^{(I)}) + 2 \gamma^{(I)}_t d_{(I)} \tilde{K}  (\tilde{w}_t - \hat{p}^{(I)}) + {\gamma_t^{(I)}}^2 (d_{(I)})^T \tilde{K} d_{(I)}
\ee
According to the Frank-Wolfe algorithm, we choose $\gamma_t^{(I)}$ which minimizes this quantity in the domain [0,1], which gives rise to:
\be\gamma_t^{(I)} = \frac{d_{(I)} \tilde{K} (\hat{p}^{(I)} - \tilde{w}_t)}{d_{(I)}\tilde{K} d_{(I)}}\ee Substituting this in (\ref{jt}) gives:
\bea
\tilde{J} (\tilde{w}_{t+1},\hat{p}^{(I)}) &=& \tilde{J} (\tilde{w}_t, \hat{p}^{(I)}) - \frac {((d_t^{(I)})^T \tilde{K} (\hat{p}^{(I)} - \tilde{w}_t))^2} {(d_t^{(I)})^T \tilde{K} d_t^{(I)}} \nonumber \\
&=& \tilde{J} (\tilde{w}_t, \hat{p}^{(I)})  \left( 1 - \inner {\frac {\frac{\sigma^{(I)}}{\tilde{\sigma}_{f_t}}l_{f_t} - L(\tilde{w}_t)} {\norm{\frac{\sigma^{(I)}}{\tilde{\sigma}_{f_t}}l_{f_t} - L(\tilde{w}_t)}}} {\frac{L^{(I)} - L({\tilde{w}_t)} } {\norm{L^{(I)} - L( {\tilde{w}_t)}}}} \right)
\eea
Now, notice that since Lemma \ref{interior} says that $L^{(I)}$ is in the relative interior of the convex polytope (in $\tilde{w}$ space) for a given $\hat{p}^{(I)}$, there exists an $r^{(I)}>0$ such that all $\tilde{w}$ that are feasible, 
\be L(\tilde{w}) + (\norm{L^{(I)} - L(\tilde{w})} + r^{(I)}) \frac {L^{(I)} - L(\tilde{w})} {\norm{L^{(I)} - L(\tilde{w})}} \ee
is also feasible. Furthermore, the Frank-Wolfe vertex $\frac{\sigma^{(I)}} {\tilde{\sigma}_{f_t}} l_{f_t}$ maximizes the inner product $\inner {L^{(I)} - L(\tilde{w}_t)} {L(\tilde{w}) - L(\tilde{w}_t)}$ over all feasible $\tilde{w}$. This implies that
\bea \label{r-bound}
\inner{\frac{\frac{\sigma^{(I)}}{\tilde{\sigma}_{f_t}}l_{f_t} - L(\tilde{w}_t)} {\norm{\frac{\sigma^{(I)}}{\tilde{\sigma}_{f_t}}l_{f_t} - L(\tilde{w}_t)}}} {\frac{L^{(i)} - L(\tilde{w}_t )} {\norm{L^{(i)} - L(\tilde{w}_t )}}} &\geq& \inner {\frac{(\norm{L^{(I)} - L(\tilde{w}_t)} + r^{(I)}) \frac {L^{(I)} - L(\tilde{w}_t)}{\norm{L^{(I)} - L(\tilde{w}_t)}}} {\norm{\frac{\sigma^{(I)}}{\tilde{\sigma}_{f_t}}l_{f_t} - L(\tilde{w}_t)}}} {\frac{L^{(I)} - L(\tilde{w}_t )} {\norm{L^{(I)} - L(\tilde{w}_t )}}} \\
&=& \frac{\sqrt{\tilde{J} (\tilde{w}_t,\hat{p}^{(I)}})+r^{(I)}}{\norm{\frac{\sigma^{(I)}}{\tilde{\sigma}_{f_t}}l_{f_t} - L (\tilde{w}_t)}} \nonumber \\
&\geq& \frac{\sqrt{\tilde{J} (\tilde{w}_t,\hat{p}^{(I)}})+r^{(I)}}{\sigma^{(I)} \bar{\eta}^{(I)}} \nonumber
\eea
Combining equations (\ref{jt}) and (\ref{r-bound}) gives:
\bea
\tilde{J} (\tilde{w}_{t+1},\hat{p}^{(I)}) &\leq& \tilde{J} (\tilde{w}_t, \hat{p}^{(I)}) \left( 1- \left( \frac{\sqrt{\tilde{J} (\tilde{w}_t,\hat{p}^{(I)}})+r^{(I)}} {\sigma^{(I)} \bar{\eta}^{(I)}} \right)^2 \right) \nonumber \\
&\leq&  {\beta^{(I)}}^2 \tilde{J} (\tilde{w}_t,\hat{p}^{(I)})  \left( 1 - \frac {\tilde{J} (\tilde{w}_t, \hat{p}^{(I)})} {(\sigma^{(I)})^2 (\bar{\eta}^{(I)})^2 (\beta^{(I)})^2}\right) \nonumber \\
&\leq& \hat{\beta}^2 \tilde{J} (\tilde{w}_t,\hat{p}^{(I)}) \left( 1 - \frac {\tilde{J} (\tilde{w}_t, \hat{p}^{(I)})} {\hat{\sigma}^2 \hat{\bar{\eta}}^2 \hat{\beta}^2} \right)
\eea
Setting $\alpha = \hat{\beta}^2$ and $x_t := \frac{\tilde{J} (\tilde{w}_t, \hat{p}^{(I)})} {\hat{\sigma}^2 \hat{\bar{\eta}}^2 \hat{\beta}^2}$.
Then by Lemma \ref{sequence},
\be \frac{\tilde{J} (\tilde{w}_t, \hat{p}^{(i)})} {\hat{\sigma}^2 \hat{\bar{\eta}}^2 \hat{\beta}^2} \leq \frac{\frac{\tilde{J} (\tilde{w}_0, \hat{p}^{(i)})} {\hat{\sigma}^2 \hat{\bar{\eta}}^2 \hat{\beta}^2} }{\hat{\beta}^{(-2t)} + t\frac{\tilde{J} (\tilde{w}_0, \hat{p}^{(i)})} {\hat{\sigma}^2 \hat{\bar{\eta}}^2 \hat{\beta}^2}}\ee
Using the fact that $\frac{a}{a+b}$ is monotonically increasing in $a$ when $a,b > 0$ and $\tilde{J} (\tilde{w}_0, \hat{p}^{(i)}) \leq \hat{\sigma}^2 \hat{\eta}^2$, one gets:
\bea 
\tilde{J} (\tilde{w}_t, \hat{p}^{(I)})  &\leq & \frac{\hat{\sigma}^2 \hat{\bar{\eta}}^2 \hat{\beta}^2} {\hat{\eta}^2\,t + \hat{\beta}^{-2(t-1)}\hat{\bar{\eta}}^2} \nonumber \\
\mathrm{or}, \;||L-L(\tilde{w}_t|| &\leq& \frac{\hat{\sigma}^2 \hat{\bar{\eta}}^2 \hat{\beta}^2}{\sqrt{\hat{\eta}^2\,(m-1) + \hat{\beta}^{-2(m-2)}\hat{\bar{\eta}}^2}}
\eea after $t=m-1$ iterations. 
\end{proof}

\subsection{Proof of Theorem \ref{theorem2} for $\mathcal{D} = \mathcal{D}^{(K)}_{CH} (\{\hat{p}^{(i)}\})_{ i \in [K]}$ }\label{th2proof}

\begin{proof}
The proof of this closely follows Theorem 5.2 in \cite{Campbell2017}. Given a posterior distribution $\nu(\theta)$, we have: 
\be \langle l_m, l_n \rangle := \mathbb{E}_{\theta \sim \nu}\,[l_n(\theta)l_m(\theta) ].\ee Given the true vector as $l_n$ and it's random projection as 
$u_n$, one can write: 
\bea
\mathbb{P}[\max_{m,n}|\langle l_m, l_n \rangle - u_m^T\,u_n| \geq \epsilon] &\leq& [\sum_{m,n}\,\mathbb{P}[|\langle l_m, l_n \rangle - u_m^T\,u_n| \geq \epsilon] \nonumber\\
&\leq& N^2\,\max_{m,n}\mathbb{P}[|\langle l_m, l_n \rangle - u_m^T\,u_n| \geq \epsilon] = N^2\,\max_{m,n}\mathbb{P}[|\langle l_m, l_n \rangle - \frac{1}{J}\sum_{j=1}^J\,u_{mj}^T\,u_{nj} | \geq \epsilon]\nonumber\\
&\leq& 2N^2\,e^{-\frac{J\epsilon^2}{2\xi^2}}, 
\eea where in the final inequality we have used Hoeffding's Lemma for the gaussian random variable $x := l_n(\theta)\,l_m(\theta)$. Thus, with probability $\geq 1-\delta$, we have:
\be \label{th2a}
\max_{m,n}|\langle l_m, l_n \rangle - u_m^T\,u_n| \leq \sqrt{\frac{2\xi^2}{J}\log{\frac{2N^2}{\delta}}}
\ee
The above can then be used to upper bound the difference between the true expectation $||L-L(\tilde{w})||_{\nu,2}$ and the approximate expectation constructed from random projections $||u -u(\tilde{w})||_{\nu,2}$, where $\{\tilde{w}, p\} = \{\tilde{w}_{m-1}^{(I)}, \hat{p}^{(I)}\}$ is the output of Algorithm 2:
\bea
|| L-L(\tilde{w})||^2_{\nu,2} - || u -u(\tilde{w})||^2_{\nu,2} &=& \sum_{m,n}\,[(\tilde{w}_m -p) (\langle l_m, l_n \rangle-u_m^T u_n)(\tilde{w}_n-p)] \nonumber\\
&\leq& \sum_{m,n}\,|\tilde{w}_m-p| |\langle l_m, l_n \rangle-u_m^T u_n| |\tilde{w}_n-p| \nonumber\\
&\leq& ||\tilde{w}-p||_1^2 \,\max_{m,n} |\langle l_m, l_n \rangle - u_m^T\,u_n|
\eea Therefore, using (\ref{th2a}), with probability greater than $1-\delta$, one has: 
\be \label{th2b}
 || L-L(\tilde{w})||^2_{\nu,2}  \leq  || u -u(\tilde{w})||^2_{\nu,2} + ||\tilde{w} - p||_1^2\,\sqrt{\frac{2\xi^2}{J}\log{\frac{2N^2}{\delta}}}
\ee
\end{proof}

\bibliographystyle{ieeetr}
\bibliography{generalization_bound}

\end{document}